\pdfoutput=1

\documentclass[11pt]{article}

\usepackage{emnlp2021}

\usepackage{times}
\usepackage{latexsym}

\usepackage[T1]{fontenc}

\usepackage[utf8]{inputenc}

\usepackage{amssymb}
\usepackage{scalefnt,letltxmacro}
\usepackage{mathtools}          
 \LetLtxMacro{\oldtextsc}{\textsc}
 \renewcommand{\textsc}[1]{\oldtextsc{\scalefont{1.10}#1}}
 \usepackage[scaled=0.92]{PTSans}
\usepackage{soul} 
\usepackage{xcolor} 
\sethlcolor{pink}
\definecolor{shadecolor}{gray}{0.9}
\usepackage[algoruled,noend]{algorithm2e}
\usepackage{listings}
\usepackage{fancyvrb}
\fvset{fontsize=\normalsize}
\usepackage{setspace}
\SetKwComment{Comment}{$\triangleright$\ }{}
\SetCommentSty{shape}
\usepackage{graphicx}
\usepackage[labelfont=bf, labelsep=period]{caption}
\usepackage[format=hang]{subcaption}
\usepackage{booktabs}
\usepackage{hyperref} 
\usepackage[all]{hypcap}
\usepackage[nameinlink]{cleveref}
\usepackage[acronym,smallcaps,nowarn]{glossaries}
\usepackage{framed}

\usepackage{amsthm}
\theoremstyle{definition}
\newtheorem{theorem}{Theorem}

\DeclareMathOperator*{\argmax}{arg\,max}
\DeclareMathOperator*{\argmin}{arg\,min}

 \newacronym{KL}{kl}{Kullback-Leibler}
\newacronym{ELBO}{elbo}{\emph{evidence lower bound}}
\newacronym{POPELBO}{pop-elbo}{\emph{population evidence lower bound}}

\newacronym{SVI}{svi}{stochastic variational inference}
\newacronym{BUMPVI}{bump-vi}{bumping variational inference}

\newacronym{GMM}{gmm}{Gaussian mixture model}
\newacronym{LDA}{lda}{latent Dirichlet allocation}

\newacronym{SUTVA}{sutva}{stable unit treatment value assumption}

\newacronym{TBIP}{tbip}{text-based ideal point model}
 
\creflabelformat{equation}{#1#2#3}

\usepackage{microtype}

\title{Rationales for Sequential Predictions}

\author{Keyon Vafa \\
  Columbia University \\
  \texttt{keyon.vafa@columbia.edu} \\\And
  Yuntian Deng \\
  Harvard University \\
  \texttt{dengyuntian@seas.harvard.edu} \\\AND
  David M. Blei \\ 
  Columbia University \\
  \texttt{david.blei@columbia.edu} \\ \And
  Alexander M. Rush \\ 
  Cornell Tech \\
  \texttt{arush@cornell.edu}
}

\begin{document}
\maketitle
\begin{abstract}

Sequence models are a critical component of modern NLP systems, but
their predictions are difficult to explain. We consider model
explanations though \textit{rationales}, subsets of context that can
explain individual model predictions. We find sequential rationales
by solving a combinatorial optimization: the best rationale is the
smallest subset of input tokens that would predict the same output as
the full sequence. Enumerating all subsets is intractable, so we
propose an efficient greedy algorithm to approximate this objective.
The algorithm, which is called greedy rationalization, applies to any
model. For this approach to be effective, the model should form
compatible conditional distributions when making predictions on
incomplete subsets of the context. This condition can be enforced
with a short fine-tuning step. We study greedy rationalization on
language modeling and machine translation. Compared to existing
baselines, greedy rationalization is best at optimizing the
combinatorial objective and provides the most faithful rationales. On a
new dataset of annotated sequential rationales, greedy rationales are
most similar to human rationales.

\end{abstract}
\section{Introduction}

Sequence models are a critical component of generation tasks ranging from
language modeling 
to machine translation 
to summarization. 
These tasks are dominated by
complex neural networks. While these models produce
accurate predictions, their decision making processes are hard to
explain. Interpreting a model's prediction is important in a variety of
settings: a researcher needs to understand a model to debug it; a
doctor using a diagnostic model requires justifications to validate a decision; 
a company deploying a language model relies on model explanations to
detect biases appropriated from training data.

\begin{figure}
  \includegraphics[width=0.448\textwidth]{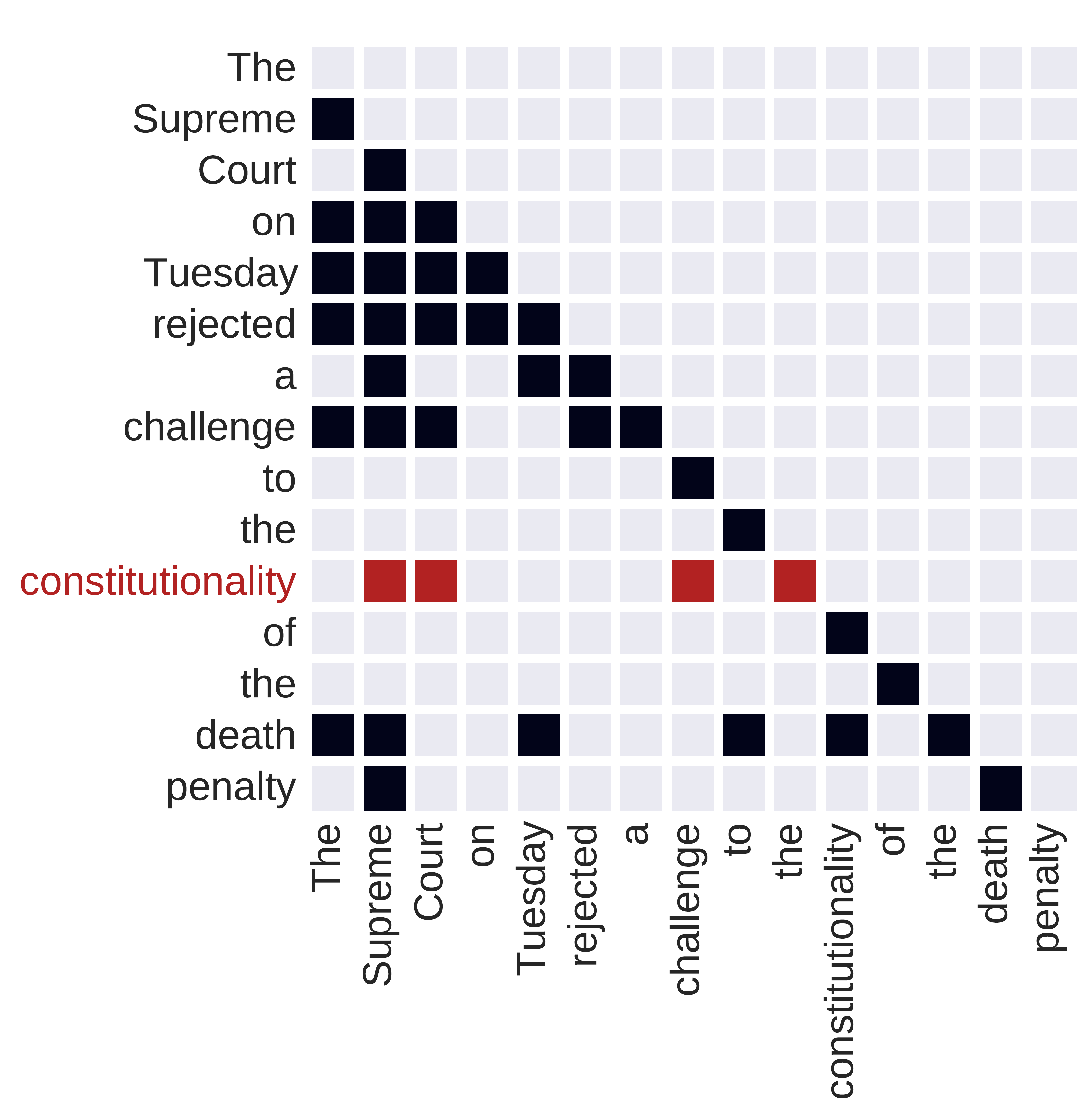}
  \caption{
  Rationales for sequential prediction on GPT-2. 
  Each row is a predicted word. The dark cells correspond to the context words found by greedy rationalization. To predict ``constitutionality'', the model only needs ``Supreme'', ``Court'', ``challenge'', and ``the''.}

  \label{fig:intro_rationale_example}
\end{figure}

Interpretation takes many flavors~\citep{lipton2018mythos}. We focus
on \textit{rationales}, i.e. identifying the most important subset of
input tokens that leads to the model's prediction.  For example,
consider the sentence: ``The Supreme Court on Tuesday rejected a
challenge to the constitutionality of the death penalty.'' Suppose we would like to explain the
decision of the model to generate ``constitutionality''. While the
model mathematically conditions on all the previous words, only some are necessary for 
its predictions. In this case, the rationale produced by our algorithm includes
``the'', ``challenge'', and notably ``Supreme Court'', but not phrases that add no information like ``on Tuesday'' (\Cref{fig:intro_rationale_example}).

Various rationale methods
have been proposed for sequence classification, where
each sequence has a single rationale \citep{lei2016rationalizing,chen2018learning,jain2020learning}. However, these methods
cannot scale to sequence models, where each token in a 
sequence requires a different rationale.

This work frames the problem of finding sequence rationales as a combinatorial
optimization: given a model, the best rationale is the smallest subset of input
tokens that would predict the same token as the full
sequence. Finding the global optimum in this setting is intractable,
so we propose \textbf{greedy rationalization}, a greedy algorithm that iteratively builds
longer rationales. This approach is efficient for many NLP models such
as transformers ~\citep{vaswani2017attention}. 
Moreover, it does not require access to the inner workings of a model, such as gradients.

Underlying this approach is an assumption that the model forms sensible predictions for incomplete subsets of the input.
Although we can pass in incomplete subsets to 
neural models, there is no guarantee that their predictions on these subsets will be compatible with their predictions on full contexts \citep{arnold1989compatible}. We show that compatibility can be learned by conditioning on randomly sampled context subsets while training a model. For large pretrained models like GPT-2 \citep{radford2019language}, fine-tuning is sufficient. 

In an empirical study, we compare greedy rationalization to various gradient- and
attention-based explanation methods on language modeling and machine
translation. Greedy rationalization best optimizes the
objective, and its rationales are most faithful to the inner workings
of the model. We additionally create a new dataset of annotated
rationales based on the Lambada corpus~\citep{paperno2016lambada}. We
find that greedy rationales are most similar to human annotations,
both on our dataset and on a labeled dataset of translation
alignments. 

Our code and annotated dataset are available.\footnote{\url{https://github.com/keyonvafa/sequential-rationales}}

 \section{Sequential Rationales}

Consider a sequence of tokens, $y_{1:T}$, generated by some unknown 
process $y_{1:T} \sim F$.
The goal of sequence modeling is to learn a probabilistic model $p_\theta$ that approximates $F$ from samples. Maximum-likelihood estimation is an effective way to train these models, where $\theta$ is fit according to
\begin{equation}
\label{eqn:setup_empirical_risk}
\argmax_\theta \mathbb{E}_{y_{1:T} \sim F}[\log p_\theta(y_{1:T})]. 
\end{equation}
Sequence models are typically factored into conditional distributions:
\begin{equation}
\label{eqn:chain_rule}
p_\theta(y_{1:T}) = f_\theta(y_1) \prod_{t=2}^T f_\theta(y_t|y_{<t}).
\end{equation}
Here,
$f_\theta$ is the specific model parameterizing $p_\theta$, such as a transformer \citep{vaswani2017attention}, and is trained to take inputs $y_{<t}$. Going forward, we drop the dependence on $\theta$ in the notation.

Word-level explanations are a natural way to interpret a
sequence model: which words were instrumental for predicting a particular word? Would the same word have been predicted if some of the words had been missing?

Explanations may be straightforward for simpler models; for example, a bigram Markov model uses only the previously generated word to form predictions. 
However, the most effective sequence models have been based on neural networks, whose predictions are 
challenging to interpret \citep{lipton2018mythos}.

Motivated by this goal, we consider a sequence $y_{1:T}$ generated by a sequence model $p$. At each position $t$, the model takes the inputs in the context $y_{<t}$ and uses them to predict $y_t$. We are interested in forming \textit{rationales}: subsets of the contexts that can explain the model's prediction of $y_t$.\footnote{Our paradigm and method extend easily to conditional sequence models, such as those used for machine translation. 
For full details, refer to \Cref{sec:appendix_alg_details}.}

What are the properties of a good rationale? Any of the contextual words $y_{<t}$ can contribute to $y_t$. However, if a model makes the same prediction with only a subset of the context, that subset contains explanatory power on its own. A rationale is \textit{sufficient} if the model would produce the same $y_t$ having seen only the rationale \citep{deyoung2019eraser}. While rationales consisting of the full context would always be sufficient, they would be ineffective for explaining longer sequences. Intuitively, the smaller the rationale, the easier it is to interpret, so we also prioritize \textit{brevity}.

We combine these desiderata and frame finding rationales as a combinatorial optimization: 
the best rationale of a word $y_t$ is the smallest subset of inputs that would lead to the same prediction. 
Each candidate rationale $S$ is an index set, and $y_S$ denotes the subset of tokens indexed by $S$. 
Denote by $\mathcal S = 2^{[t-1]}$ the set of all possible context subsets. An optimal rationale is given by
\begin{equation}
\label{eqn:combinatorial_rationale}
\argmin_{S \in \mathcal{S}} |S| \ \  \text{s.t.} \ \ \argmax_{y_t'} p(y_t' | y_{S}) = y_t.
\end{equation}
The constraint guarantees sufficiency, and the objective targets brevity. Although the objective may have multiple solutions, we only require one.

Optimizing \Cref{eqn:combinatorial_rationale} is hindered by a pair of
computational challenges. The first challenge is that solving this
combinatorial objective is intractable; framed as a decision
problem, it is NP-hard. We discuss this challenge in
\Cref{sec:greedy}. The second challenge is that evaluating distributions
conditioned on incomplete context subsets $p(y_t'|y_S)$ involves an intractable
marginalization over 
missing tokens. For now we assume that $f(y_t'|y_S) \approx p(y_t'|y_S)$; we discuss how to enforce this condition in \Cref{sub:modularity}.

 \section{Greedy Rationalization}
\label{sec:greedy}

\begin{figure*}
   \makebox[\textwidth][c]{\includegraphics[width=0.95\textwidth]{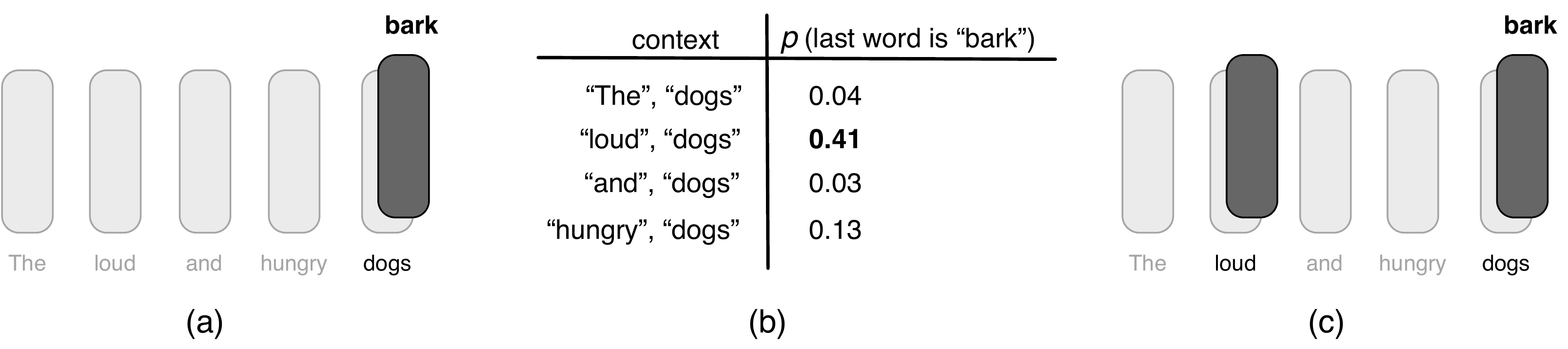}}
  \caption{One step of greedy rationalization. In (a), the rationale so far is a single word, ``dogs''. In (b), each candidate token is considered and ``loud'' results in the best probability for ``bark''. In (c), the token ``loud'' is added to the rationale. This process repeats until the most likely word is the  model prediction.}
  \label{fig:greedy_search}
\end{figure*}

We propose a simple greedy algorithm, \textbf{greedy rationalization}, to approximate the solution to \Cref{eqn:combinatorial_rationale}. 
The algorithm starts with an empty rationale. At each step, it considers adding each possible token, and it selects the one that most increases the probability of $y_t$. This process is repeated until the rationale is sufficient for predicting $y_t$. 
\Cref{fig:greedy_search} provides an overview. 

Here is the algorithm. Begin with a rationale $S^{(0)} = \emptyset$. Denoting by $[t-1] = \{1,\dots, t-1\}$, the first rationale set is
\begin{equation}
\label{eqn:distillation_step_1}
S^{(1)} = \argmax_{k \in [t-1]} p(y_t | y_k).
\end{equation}
At each step, we iteratively add a single word to the rationale, choosing the one that maximizes the probability of the word $y_t$:
\begin{equation}
\label{eqn:distillation_step_n}
S^{(n+1)} = S^{(n)} \ \cup \argmax_{k \in [t-1] \setminus S^{(n)}} p(y_t | y_{S^{(n)} \cup k}).
\end{equation}
We continue iterating \Cref{eqn:distillation_step_n} until $\argmax_{y'_t} p(y'_t | y_{S^{(n)}}) = y_t$. The procedure will always converge, since in the worst case, $S^{(t-1)}$ contains the full context.

The greedy approach is motivated by approximations to the set cover problem~\citep{chvatal1979greedy}. In our setting, each set is a single context token, and a rationale ``covers'' a sequence if it results in predicting the same token.

This procedure is simple to implement, and it is black-box: it does not require access to the inner workings of a model, like gradients or attention.

While greedy rationalization can be applied to any model, greedy
rationalization is particularly effective for set-based models such as
transformers. If we assume the rationale size $m=|S|$ is significantly
shorter than the size of the context $t$, 
greedy rationalization 
requires no extra asymptotic complexity beyond the cost of a single evaluation.

For transformers, the complexity of each evaluation $f(y_t| y_{<t})$ is quadratic in the input set $O(t^2)$. Each step of greedy rationalization requires evaluating $f(y_t| y_{S})$, but $y_S$ can be significantly smaller than $y_{<t}$. A rationale of size $m$ will require $m$ steps of $O(t)$ evaluations to terminate, resulting in a total complexity of $O(m^3t)$. As long as $m = O(t^{1/3})$, greedy rationalization 
can be performed with the same asymptotic complexity as evaluating a transformer on the full input, $O(t^2)$. 
In \Cref{sec:appendix_computational_comparisons}, we verify the efficiency of greedy rationalization.

 \section{Model Compatibility}
\label{sub:modularity}

Greedy rationalization requires computing conditional distributions $p(y_t|y_S)$ for arbitrary subsets $S$. Using an autoregressive model, this calculation requires marginalizing over unseen positions. For example, rationalizing a sequence $y_{1:3}$ requires evaluating the candidate rationale $p(y_3|y_1)$, which marginalizes over the model's predictions:
\begin{align*}
\label{eqn:marginalization_2}
p(y_3|y_1) = \sum_k f(y_3| y_1, y_2=k) f(y_2=k|y_1).
\end{align*}

Given the capacity of modern neural networks, it is tempting to pass in incomplete subsets $y_S$ to $f$ and evaluate this instead as $f(y_t|y_S) \approx p(y_t|y_S)$. However, since $f$ is trained only on complete feature subsets $y_{<t}$, incomplete feature subsets $y_S$ are out-of-distribution \citep{hooker2018benchmark}. Evaluating $f(y_3| y_1)$ may be far from the true conditional $p(y_3|y_1)$. In \Cref{fig:gpt_the}, we show that indeed language models like GPT-2 produce poor predictions on incomplete subsets.

\subsection{Fine-tuning for Compatibility}
Ideally $f(y_t|y_S)$ approximates $p(y_t|y_S)$, a property known as \textit{compatibility} \citep{arnold1989compatible}. Since training with \Cref{eqn:setup_empirical_risk} only evaluates $f$ on complete contexts $y_{<t}$, its behavior on incomplete contexts $y_S$ is unspecified. Instead, compatibility can be obtained 
by training to maximize
\begin{equation}
\label{eqn:word_dropout_objective}
\mathbb{E}_{y_{1:T} \sim F} \mathbb{E}_{S \sim \text{Unif}(\mathcal{S})} \left[\textstyle\sum_{t=1}^T \log f(y_t | y_{S_{<t}})\right],
\end{equation}
where $S \sim \text{Unif}(\mathcal S)$ indicates sampling word subsets uniformly at random from the power set of all possible word subsets, and $S_{<t}$ denotes the indices in $S$ that are less than $t$. 
\citet{jethani2021have} show that the optimum of \Cref{eqn:word_dropout_objective} is the distribution whose conditional distributions are all equal to the ground-truth conditionals.

We approximate \Cref{eqn:word_dropout_objective} with word dropout. In practice, we combine this objective with standard MLE training to learn compatible distributions while maintaining the performance of the original model. The word dropout distribution in \Cref{eqn:word_dropout_objective} is heavily skewed towards contexts containing half the words in the sequence. To alleviate this problem, we modify the word dropout distribution to sample subsets of varying lengths;
see \Cref{sec:appendix_finetuning}.

The intuition for \Cref{eqn:word_dropout_objective} is straightforward: if the model sees incomplete contexts while training, it can approximate arbitrary incomplete distributions. Since $f(y_t|y_S)$ approximates $F(y_t|y_S)$ and $f(y_t|y_{<t})$ approximates $F(y_t|y_{<t})$, all the conditional distributions are compatible.

\begin{figure}
   \includegraphics[width=0.47\textwidth]{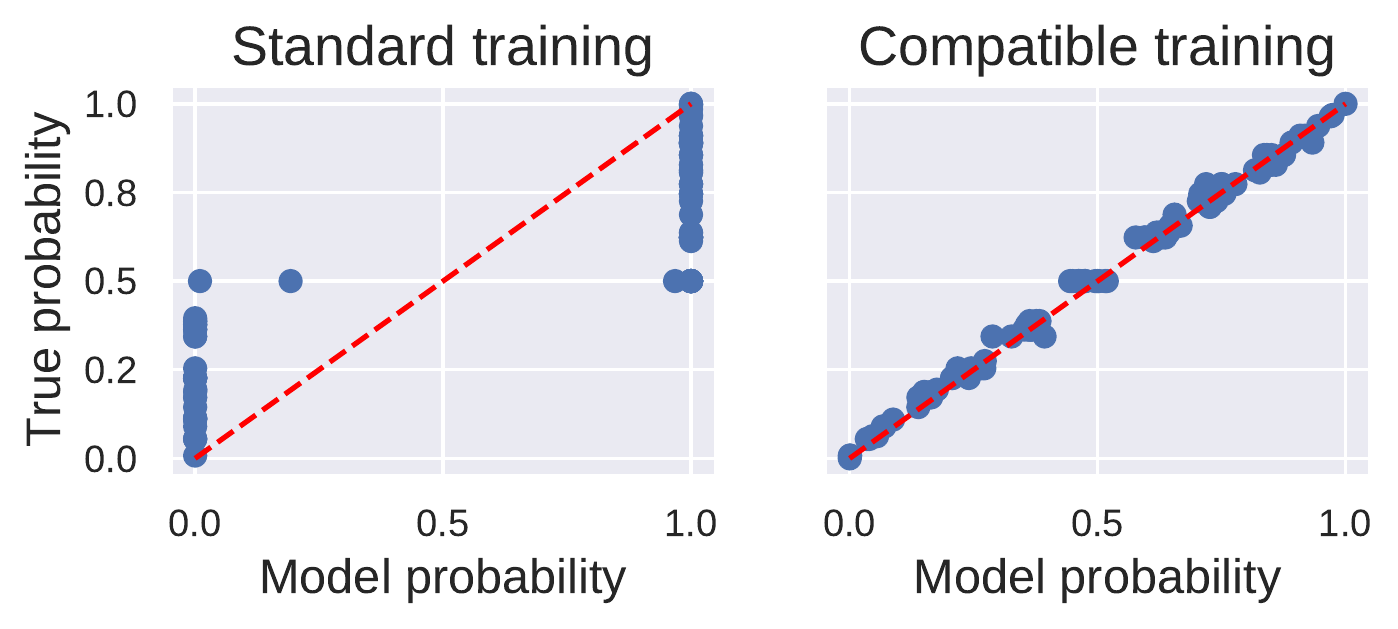}
  \caption{Training with word dropout (right) results in compatible predictions for the majority-class synthetic language. The optimal compatibility is the 
  dashed line.}
  \label{fig:majority_class_calibration}
\end{figure}

\subsection{Compatibility Experiments}

To demonstrate the impact of training with the compatibility objective in \Cref{eqn:word_dropout_objective}, we consider a
synthetic majority-class language over binary strings of 19 tokens. The first 17 are
sampled uniformly from $\{0, 1\}$, and the 18th token is always
`='. The 19th token is $0$ if there are more $0$'s than $1$'s in the
first 17 tokens, and $1$ otherwise.

We train two models: one using the standard objective in \Cref{eqn:setup_empirical_risk}, the other using word dropout to optimize \Cref{eqn:word_dropout_objective}.
Although both models have the same heldout perplexity on the full context, training with \Cref{eqn:word_dropout_objective} is required to form compatible predictions on incomplete subsets. 
In \Cref{fig:majority_class_calibration}, we provide both models with random subsets $S$ and calculate each model's probability that
the last token is $1$. 
A model
that has only seen a few tokens should be less confident about the
prediction of the final majority class, yet models trained without word dropout ignore this uncertainty.

Models do not need to be trained from scratch with \Cref{eqn:word_dropout_objective}. A model can be pre-trained with \Cref{eqn:setup_empirical_risk}, after which it can be fine-tuned for compatibility. As an example, when GPT-2 is not trained with word dropout, it makes insensible predictions for out-of-distribution sequences. For a sequence that contains only the token ``the'', GPT-2 is trained to give reasonable predictions for $p(y_2| y_{1} = \text{``the''})$. But when it has only seen the token ``the'' somewhere besides the first position of the sequence, the top prediction for the word after ``the'' is also ``the''.\footnote{We represent ``the'' at various positions by changing the positional encoding passed into the transformer.} 
Of course, following ``the'' with ``the'' is not grammatical. 
Fine-tuning for compatibility alleviates this problem (\Cref{fig:gpt_the}).

\begin{figure}
  \includegraphics[width=0.47\textwidth]{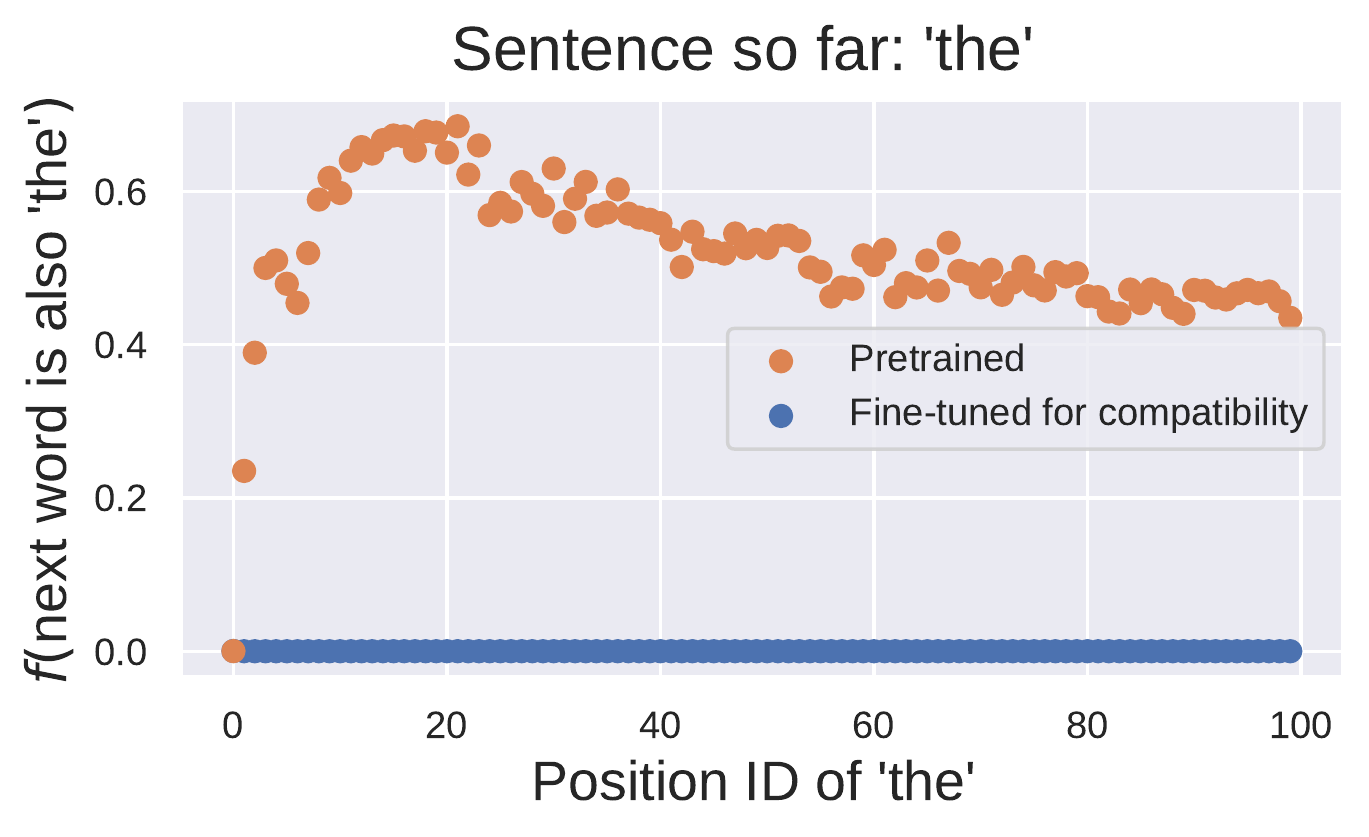}
  \caption{Fine-tuning GPT-2 for compatibility removes pathological repeating on incomplete contexts. For a position $t$, the vertical axis gives $f(y_{t+1}= \text{``the''} | y_{t} = \text{``the''})$.} 
  \label{fig:gpt_the}
\end{figure}

Finally, we find that that fine-tuning for compatibility does not hurt the heldout performance of the complete conditional distribution of each fine-tuned model (see \Cref{sec:appendix_finetuning}).

\section{Connection to Classification Rationales}

In this section, we go over related rationalization approaches developed for 
classification and discuss why they cannot scale to sequence models.
We also show that the combinatorial rationale objective in \Cref{eqn:combinatorial_rationale} is a global solution to a classification rationale-style objective.

In classification problems, a sequence $x_{1:T}$ is associated with a label $y$.
Rationale methods are commonly used in this setting \citep{lei2016rationalizing,chen2018learning,yoon2018invase,bastings2019interpretable,jain2020learning,jethani2021have}.
The most common approach uses two models: one, a selection model $q(S|x_{1:T})$, provides a distribution over possible rationales; the other, the predictive model $p(y|x_S)$, makes predictions using only samples from the former model. Typically, $p$ and $q$ are both optimized to maximize
\begin{equation}
\label{eqn:general_rationale}\mathbb{E}_{x,y \sim F}\mathbb{E}_{S \sim q(S|x,y)}[\log p(y|x_S) - \lambda |S|].
\end{equation}
Here, $F$ is the ground truth, unknown data distribution, and $\lambda$ is a regularizing penalty that encourages smaller rationales.

In practice, it is infeasible to adopt this objective for sequence models. \Cref{eqn:general_rationale} is centered on providing classification models with only the words in a sequence's rationale. In sequential settings, each word has a different rationale. Since sequence models make $T$ predictions per sequence and are trained by sharing all $T$ word representations, each token would be indirectly exposed to words in the rationales of the words it is allowed to use. A remedy would be to train sequence models without sharing representations, but this is computationally infeasible; it requires $O(T^3)$ computations per sequence for transformer architectures.

Most classification rationale methods treat $q(S|x_{1:T})$ as a probability distribution over all possible rationales. However, the $q$ that maximizes \Cref{eqn:general_rationale} is deterministic for any $p$. To see this, note that $q$ does not appear inside the expectation in \Cref{eqn:general_rationale}, so it can place all its mass on a single mode. We provide a formal justification in \Cref{sec:appendix_deterministic_proof}.

Since the optimal selection model $q$ is a point-mass, the optimal rationale can be written as
\begin{equation}
\label{eqn:soft_contrained_search}
\argmin_{S \in \mathcal S} \ \lambda |S| - \log p(y | x_S).
\end{equation}
This optimization is identical to the combinatorial optimization in \Cref{eqn:combinatorial_rationale}, albeit with a soft constraint on the rationale's prediction: the true label $y$ is not required to be the maximum of $p(y'|x_S)$. 
In practice, this soft constraint sometimes results in 
empty rationales \citep{jain2020learning}. Since we view sufficiency as a key component of a good rationale, \Cref{eqn:combinatorial_rationale} imposes a hard constraint on the rationale's prediction.

 \section{Related Work}
\label{sec:relwork}

Finding rationales is similar to feature selection. While global feature selection has been a well-studied problem in statistics \citep{guyon2003introduction,hastie2009elements}, 
instance-wise feature selection --- where the goal is selecting features per-example --- is a newer research area \citep{chen2018learning}. We review local explanation methods used for NLP.

\paragraph{Gradient saliency.}

Gradient-based saliency methods have long been used as a measure of feature importance in machine learning \citep{baehrens2010explain,simonyan2013deep,li2015visualizing}. Some variations involve word embeddings \citep{denil2014extraction}; integrated gradients, to improve sensitivity \citep{sundararajan2017axiomatic}; and relevance-propagation to track each input's contribution through the network \citep{bach2015pixel,voita2020analyzing}. 

But there are drawbacks to using gradient-based methods as explanatory tools. \citet{sundararajan2017axiomatic} show that in practice, gradients are \textit{saturated}: they may all be close to zero for a well-fitted function, and thus not reflect importance. 
Adversarial methods can also distort gradient-based saliences while keeping a model's prediction the same 
\citep{ghorbani2019interpretation,wang2020gradient}. We compare greedy rationalization to gradient saliency methods in \Cref{sec:results}.

\paragraph{Attention.}

Recently, NLP practitioners have focused on using attention weights as explanatory tools. 
The literature has made a distinction between \textit{faithfulness} and \textit{plausibility}. An explanation is faithful if it accurately depicts how a model makes a decision \citep{jacovi2020towards}; an explanation is plausible if it can be understood and interpreted by humans \citep{wiegreffe2019attention}. Practitioners have shown that attention-based explanations are generally not faithful \citep{jain2019attention,serrano2019attention}, but that they may be plausible \citep{wiegreffe2019attention,mohankumar2020towards,vashishth2019attention}. Others show that attention weights should not be interpreted as belonging to single tokens since they mix information across tokens 
\citep{brunner2019identifiability,kobayashi2020attention}. \citet{bastings2020elephant} argue that general input saliency measures, such as gradients, are better suited for explainability than attention. 
We compare greedy rationalization to attention-based methods in \Cref{sec:results}.

\paragraph{Local post-hoc interpretability.}
Another class of methods provides local interpretability for pretrained models. These approaches aim to explain a model's behavior for a single example or for a small subset of inputs. LIME \citep{ribeiro2016should} trains an interpretable model that locally approximates the pretrained model. \citet{alvarez2017causal} learn a causal relationship between perturbed inputs and their model outputs. 
These methods impose no constraints on the pretrained model. 
However, they are expensive -- they require training separate models for each input region. In contrast, the method proposed here, greedy rationalization, can efficiently explain many predictions.

\paragraph{Input perturbation.}

Practitioners have also measured the importance of inputs by perturbing them \citep{zeiler2014visualizing,kadar2017representation}. Occlusion methods \citep{li2016understanding} replace an input with a baseline (e.g. zeros), while omission methods \citep{kadar2017representation} remove words entirely. \citet{li2016understanding} propose a reinforcement learning method that aims to find the minimum number of occluded words that would change a model's prediction.
\citet{feng2018pathologies} use gradients to remove unimportant words to see how long it takes for the model's prediction to change. They find that the remaining words are nonsensical and do not comport with other saliency methods. 
Others have shown that input perturbation performs worse than other saliency methods in practice \citep{poerner2018evaluating}. 
These methods have mostly focused on subtractive techniques. For this reason, they are inefficient and do not aim to form sufficient explanations. In contrast, greedy rationalization efficiently builds up sufficient explanations.

 \section{Experimental Setup}

\label{sec:empirical_studies}

There are two goals in our empirical studies. The first is to compare the ability of greedy rationalization to other approaches for optimizing the combinatorial objective in \Cref{eqn:combinatorial_rationale}. The second is to assess the quality of produced rationales.

We measure the quality of rationales using two criteria: faithfulness and plausibility. 
An explanation is faithful if it accurately depicts how a model makes a decision \citep{jacovi2020towards}; an explanation is plausible if it can be understood and interpreted by humans \citep{wiegreffe2019attention}.
Although sufficiency is a standard way to measure faithfulness \citep{deyoung2019eraser}, all the rationales that satisfy the constraint of \Cref{eqn:combinatorial_rationale} are sufficient by definition. To measure plausibility, we compare rationales to human annotations. Since there do not exist language modeling datasets with human rationales, we collected annotations based on Lambada ~\citep{paperno2016lambada}. The annotated dataset is available online, along with the code used for all experiments.\footnote{\url{https://github.com/keyonvafa/sequential-rationales}}

We compare greedy rationalization to a variety of gradient- and attention-based baselines (see \Cref{sec:relwork}).
To form baseline sequential rationales, we add words by the order prescribed by each approach, stopping when the model prediction is sufficient. The baselines are: $l_2$ gradient norms of embeddings~\citep{li2015visualizing}, embedding gradients multiplied by the embeddings~\citep{denil2014extraction}, integrated gradients~\citep{sundararajan2017axiomatic}, attention rollout~\citep{abnar2020quantifying}, the last-layer transformer attention weights averaged-across heads, and all transformer attentions averaged across all layers and heads~\citep{jain2020learning}.

To compare rationale sets produced by each method to those annotated by humans, we use the set-similarity metrics described in \citet{deyoung2019eraser}: the intersection-over-union (IOU) of each rationale and the human rationale, along with the token-level F1, treating tokens as binary predictions (either in the human rationale or out of it).

We use transformer-based models for all of the experiments. We fine-tune each model for compatibility using a single GPU. That we can fine-tune GPT-2 Large \citep{radford2019language} to learn compatible conditional distributions on a single GPU suggests that most practitioners will be able to train compatible models using a reasonable amount of computation. For model and fine-tuning details, refer to \Cref{sec:appendix_finetuning}.

\section{Results and Discussion}
\label{sec:results}

The experiments test sequential rationales for language modeling and
machine translation. \Cref{sec:appendix_experimental_details} contains full details for each experiment.

\subsection{Language Modeling}

\paragraph{Long-Range Agreement.}
The first study tests whether rationales for language models
can capture long-range agreement. We create a template dataset using
the analogies from \citet{mikolov2013representations}. This dataset
includes word pairs that contain either a semantic or syntactic
relationship. For each type of relationship, we use a predefined template. It
prompts a language model to complete the word pair after it has seen
the first word.

For example, one of the fifteen categories is countries and their
capitals. We can prompt a language model to generate the capital by
first mentioning a country and then alluding to its capital. To test
long-range agreement, we also include a distractor sentence that
contains no pertinent information about the word pair. For example,
our template for this category is,

\begin{quote}
  \small
When my flight landed in \textbf{Japan}, I converted my currency and slowly fell asleep. (I had a terrifying dream about my grandmother, but that's a story for another time). I was staying in the capital, \_\_\_\_\_\_\_\_\_\_\_\_\_\_\_
\end{quote}
Here, the parenthetical clause is a distractor sentence, since it contains no relevant information about predicting the capital of Japan. The correct capital, ``Tokyo'', is predicted by GPT-2 
both with and without the distractor. We use this template for all of the examples in the country capital category, swapping the antecedent ``Japan'' for each country provided in \citet{mikolov2013representations}.

We feed the prompts to GPT-2, which completes each analogy.
To measure faithfulness, we calculate the percent of rationales that
contain the true antecedent, and the percent of rationales that do not
contain any words in the distractor. We only use examples where the
prediction is the same both with and without the distractor. We also
perform exhaustive rationale search on the objective in
\Cref{eqn:combinatorial_rationale}. This search is highly inefficient, so we
only complete it for 40 examples. To measure the approximation ratio,
we divide the size of the rationale found by each method by the 
exhaustive rationale size.

\Cref{tab:agreement_comparisons} contains the results on the compatible model. Although all methods contain the true antecedents in their rationales, greedy rationalization has by far the least distractors in its rationales. The rationales are also universally shorter for greedy rationalization and closer to the optimal rationales, justifying our greedy assumption. To show that fine-tuning GPT-2 for compatibility is not hurting the baselines, we also perform the baseline methods on a pretrained GPT-2 without fine-tuning; see \Cref{sec:appendix_experimental_details}.

\begin{table}
  \small
  \centering
  \begin{tabular}{ l  c c c  c }
  \toprule
  & Length & Ratio & Ante & No D \\
    \midrule
  Grad norms         & 22.5 & 4.1 & \textbf{1.0} & 0.06 \\
  Grad x emb         & 38.0 & 7.4 & 0.99 & 0.01 \\
  Integrated grads   & 28.1 & 5.2 & 0.99 & 0.00 \\
  Attention rollout  & 36.9 & 7.1 & \textbf{1.0} & 0.12 \\
  Last attention     & 16.7 & 2.9 & 0.99 & 0.13 \\
  All attentions     & 14.5 & 2.6 & \textbf{1.0} & 0.02 \\
  Greedy             & \textbf{7.1} & \textbf{1.2} & \textbf{1.0} & \textbf{0.43}  \\
    \bottomrule
 \end{tabular}
 \caption{Language modeling faithfulness on long-range agreement with templated analogies.  ``Ratio'' refers to the approximation ratio of each method's rationale length to the exhaustive search minimum. ``Ante'' refers to the percent of rationales that contain the true antecedent. ``No D'' refers to the percent of rationales that do not contain any tokens from the distractor.
  }
 \label{tab:agreement_comparisons}
\end{table}

\paragraph{Annotated Rationales.}

To test the plausibility of rationales for language models, we collect a dataset of human annotations. We base the collection on Lambada \citep{paperno2016lambada}, a corpus of narrative passages. Each passage included in Lambada is chosen so that humans need to use both local and global context to reliably predict the final word.
By its construction it is guaranteed to have non-trivial rationales.

Our goal is to collect rationales that are both minimal and sufficient for humans. We run an annotation procedure with two roles: a selector and a predictor. First, the selector sees the full passage and ranks the words in order of how informative they are for predicting the final word. Next, the predictor sees one word at a time chosen by the selector, and is asked to predict the final word of the passage. The words the predictor saw before guessing the correct word form a human rationale.
This rationale selection method is inspired by Rissanen Data Analysis \citep{rissanen1978modeling,perez2021rissanen}, which uses a minimum description length metric to estimate feature importances. 
We rely on human annotators to estimate information gains.

Since it could be trivial for humans to predict the final word if it also appears in the context, we 
only include examples that do not repeat a word. We collect annotations for 107 examples, which we also release publicly. We use two sets of annotators for 15\% of the examples in order to compute inter-annotator agreement. On this subset, the average token-level Cohen's $\kappa$ is 0.63 \citep{cohen1960coefficient}, indicating substantial agreement.

\begin{table}
  \small
  \renewcommand{\arraystretch}{1.0}  
  \centering
  \begin{tabular}{ l  c c c c}
   \toprule
   & Length & IOU & F1 \\
    \midrule
  Gradient norms           & 60.2 & 0.14 & 0.22 \\
  Gradient x embedding     & 68.3 & 0.12 & 0.21 \\
  Integrated gradients     & 62.8 & 0.12 & 0.21 \\
  Attention rollout        & 73.9 & 0.11 & 0.19 \\
  Last attention layer     & 54.6 & 0.15 & 0.25 \\
  All attention layers     & 48.7 & 0.20 & 0.28 \\
  Greedy                   & \textbf{17.9} & \textbf{0.25} & \textbf{0.35} \\
    \bottomrule
 \end{tabular}
 \caption{Language modeling plausibility on rationale-annotated Lambada.}
 \label{tab:lambada_comparisons}
\end{table}

We compare the rationales produced by each method to the annotated rationales. \Cref{tab:lambada_comparisons} shows that the greedy rationales are most similar to the human-annotated rationales. Greedy rationalization is also the most effective at minimizing the combinatorial objective in \Cref{eqn:combinatorial_rationale}, as its rationales are by far the shortest. \Cref{fig:lambada_example} contains examples of rationales for this dataset.

\begin{figure}
  \includegraphics[width=0.48\textwidth]{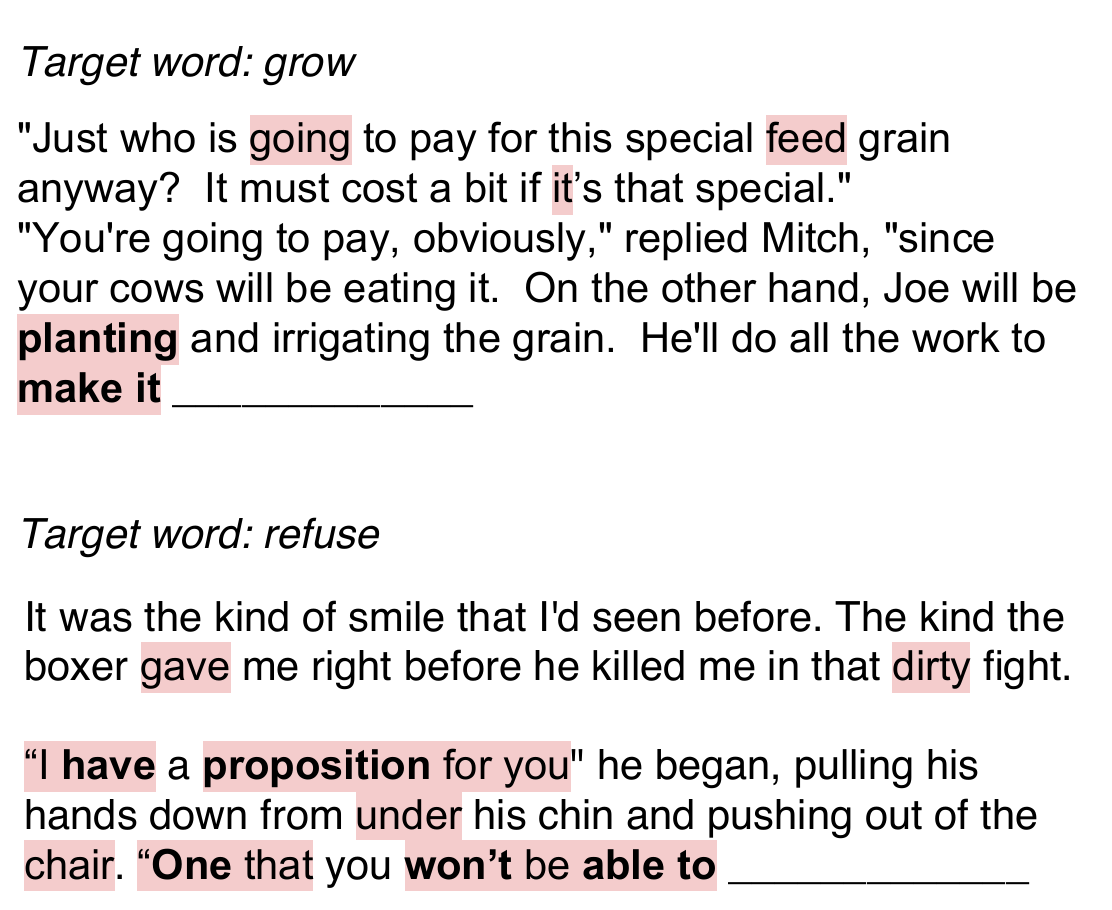}

  \caption{Examples from our annotated Lambada dataset. \hl{Highlighted text} denotes greedy rationales, and \textbf{bolded text} denotes human-annotated rationales.}
  \label{fig:lambada_example}
\end{figure}

It is worth noting that the top few words added by the baselines are quite relevant; after 5 tokens, the all-attention baseline has a better F1 and IOU than greedy rationalization. However, the baselines struggle to form sufficient rationales, which hurts their overall performance.

\subsection{Machine Translation}

\paragraph{Distractors.}
To measure faithfulness, we take a transformer trained on IWSLT14 De-En (and fine-tuned for compatibility), and generate translations for 1000 source sequences from the test set. We then create a corpus by concatenating random example pairs; for two sampled pairs of source and target sequences, $(S_1, T_1)$ and $(S_2, T_2)$, we create a new example $(S_1S_2, T_1T_2)$. Each token in $T_1$ is generated from $S_1$ alone, so its rationales shouldn't contain any tokens from $S_2$. Similarly, $T_2$ is generated from $S_2$ alone, so its rationales shouldn't contain any tokens from $S_1$ or $T_1$. 

We evaluate each rationale by counting how many times it has crossed over: a rationale for $T_1$ crosses over every time it contains a token in $S_2$, and a rationale for $T_2$ crosses over every time it contains a token in $S_1$ or $T_1$ (since the model is autoregressive, $T_1$'s rationales can never contain tokens from $T_2$).

\begin{table}
  \small
  \centering
  \begin{tabular}{ l  c c  c c }
  \multicolumn{1}{c}{} & \multicolumn{2}{c}{\textbf{Mean Crossovers}} &\multicolumn{2}{c}{\textbf{Crossover Rate}} \\  \toprule
          & Source & Target & Source & Target \\
    \midrule
  Grad norms         & 0.40 & 0.44 & \textbf{0.06} & 0.06  \\
  Grad x emb         & 6.25 & 5.57 & 0.42 & 0.41 \\
  Integrated grads   & 2.08 & 1.68 & 0.23 & 0.14 \\
  Last attention     & 0.63 & 2.41 & 0.09 & 0.24 \\
  All attentions     & 0.58 & 0.80 & 0.08 & 0.12 \\
  Greedy             & \textbf{0.12} & \textbf{0.12} & 0.09 & \textbf{0.02} \\
    \bottomrule
 \end{tabular}
 \caption{Translation faithfulness with distractors. ``Mean crossovers'' refers to the average number of crossovers per rationale, and ``Crossover rate'' refers to the fraction of rationales that contain at least one crossover.}
 \label{tab:ghost_comparisons}
\end{table}

\Cref{tab:ghost_comparisons} contains the results. Greedy rationalization has by far the fewest average number of crossovers per rationale. Although the percent of source rationales that cross over is slightly higher than the percent using gradient norms, the percentage on the target side is superior. 

\paragraph{Annotated Alignments.}
To test plausibility, we compare the rationales to word alignments \citep{brown1993mathematics}. Using a dataset containing 500 human-labeled alignments for German-English translation,\footnote{\url{https://www-i6.informatik.rwth-aachen.de/goldAlignment/}}
we compute rationales for each method using the ground truth targets. We measure similarity to the labeled rationales by computing alignment error rate (AER)~\citep{och2000improved}, along with computing the IOU and F1 between sets. To separate the requirement that the rationale be sufficient from each method's global ordering of tokens, we also compare top-1 accuracies, which measure whether the top token identified by each baseline is present in the labeled alignment set.

\begin{table}
  \small
  \centering
  \begin{tabular}{ l  c c c c  c}
  \toprule
          & Length & AER & IOU & F1 & Top1\\
    \midrule
  Grad norms        & 10.2 & 0.82 & 0.30 & 0.16   & 0.63\\
  Grad x emb & 13.2 & 0.90 & 0.16 & 0.12 & 0.40 \\
  Integrated grads & 11.3 & 0.85 & 0.24 & 0.14 & 0.42 \\
  Last attention & 10.8 & 0.84 & 0.27 & 0.15 & 0.59\\
  All attentions & 10.7 & 0.82 & 0.32 & 0.15 & \textbf{0.66} \\
  Greedy          & \textbf{4.9} & \textbf{0.78} & \textbf{0.40} & \textbf{0.24} & 0.64\\
    \bottomrule
 \end{tabular}
 \caption{
   Translation plausibility with annotated alignments.
   The first four columns correspond to using the full source rationale found by each method; the last column ``Top1'' refers to the accuracy of the first source token added by each  method. 
 AER refers to alignment error rate.}
 \label{tab:mt_comparisons}
\end{table}

\begin{figure}
  \includegraphics[width=0.47\textwidth]{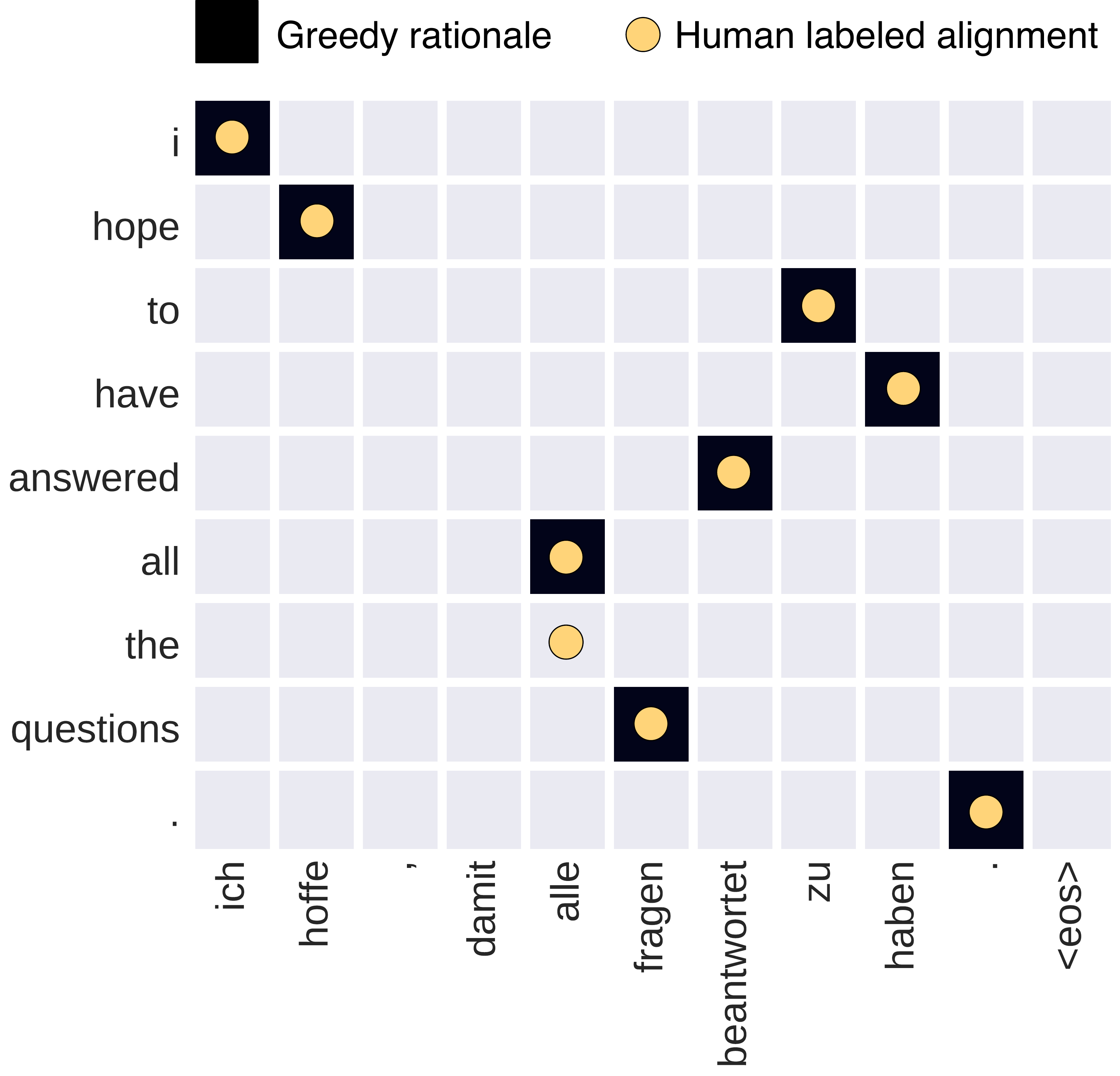}

  \caption{Greedy rationalization for machine translation. Each row depicts the source words contained in a rationale. Although each rationale includes both source and target words, here we only show source-side rationales so they can be compared to annotated alignments.}
  \label{fig:mt_example}
\end{figure}

\Cref{tab:mt_comparisons} contains the results. 
The rationales learned by greedy rationalization are more similar to human-labeled alignments than those provided by gradient and attention methods. Many methods have similar top-1 accuracies --- indeed, the best top-1 accuracy comes from averaging all attention layers. This reinforces the notion that although the baselines may be able to capture first-order information, they struggle to form sufficient rationales. \Cref{fig:mt_example} contains an example of greedy rationalization applied to machine translation, along with the human-labeled alignments.

 \section{Conclusion}
\label{sec:conclusion}

We proposed an optimization-based algorithm for rationalizing sequence predictions. Although exact optimization is intractable, we developed a greedy approach that efficiently finds good rationales. Moreover, we showed that models can be fine-tuned to form compatible distributions, thereby circumventing an intractable marginalization step. In experiments, we showed that the greedy algorithm is effective at optimization, and that its rationales are more faithful and plausible than those of gradient- and attention-based methods. 
We hope that our research, along with the release of an annotated dataset of sequence rationales, catalyzes further research into this area.

\paragraph{Acknowledgments} This work is funded by ONR N00014-17-1-2131, ONR N00014-15-1-2209, NSF CCF-1740833, DARPA SD2 FA8750-18-C-0130, Two Sigma, Amazon, and NVIDIA. Keyon Vafa is supported by the Cheung-Kong Innovation Doctoral Fellowship. Alexander Rush and Yuntian Deng are sponsored by NSF 1901030 and NSF CAREER 2037519. We also thank Mark Arildsen, Elizabeth Chen, Justin Chen, Katherine Chen, Nathan Daniel, Alexander Hem, Farzan Vafa, Neekon Vafa, Willy Xiao, and Carolina Zheng.

\bibliographystyle{acl_natbib}
\interlinepenalty=10000
\bibliography{draft}

\clearpage

\appendix

\section{Algorithm Details}
\label{sec:appendix_alg_details}
We present greedy rationalization in \Cref{alg:greedy_rationalization}.

\begin{algorithm}
\SetArgSty{textup}
\DontPrintSemicolon
\setstretch{1.08}
\KwIn{Sequence $y_{1:t}$ generated from $p$.}
\KwOut{Rationale $S$ for $y_t$.}
\textbf{Initialize:} $S = \emptyset$\\
\While{$\argmax_{y_t'} p(y_t'|y_S) \neq y_t$} {
  $k^* = \argmax_{k \in [t-1]\setminus S} \ p(y_t|y_{S\cup k})$\\
  $S = S \cup k^*$
}
\Return{$S$}
\caption{Greedy rationalization}
\label{alg:greedy_rationalization}
\end{algorithm}

Most sequence models, including transformers, use the representation of a token $y_{t-1}$ to predict the next token, $y_t$. As such, a rationale $S$ always needs to contain $y_{t-1}$. In practice, we initialize $S = \{y_{t-1}\}$.

This method and paradigm extend easily to conditional sequence models, such as those used in machine translation. In this setting, a model uses a source sequence $x_{1:N}$ to generate a target sequence $y_{1:T}$. Thus, a context for a prediction $y_t$ contains both $y_{<t}$ and $x_{1:N}$. The set of all possible rationales is the cross product of power sets $\mathcal{S} = 2^{[N]} \times 2^{[t-1]}$, and the combinatorial objective is
\begin{align*}
\label{eqn:appendix_conditional_combinatorial_rationale}
S(x_{1:N}, y_{1:t}) &= \argmin_{S_x, S_y \in \mathcal{S}} |S_x| + |S_y| \ \  \\
&\text{s.t.} \ \ \argmax_{y_t'} p(y_t' | x_{S_x}, y_{S_y}) = y_t.
\end{align*}

To perform greedy rationalization in this setting, we consider adding either a source token or a target token at each step, choosing the one that results in the largest increase in the full model's prediction.

\section{Optimality of Deterministic Rationales}
\label{sec:appendix_deterministic_proof}
Here, we show that the selection distribution $q(S|x,y)$ that maximizes the classification rationale objective in \Cref{eqn:general_rationale} is deterministic. We re-write the objective below:
\begin{equation}
\label{eqn:appendix_general_rationale}
\mathbb{E}_{x,y \sim F}\mathbb{E}_{S \sim q(S|x,y)}[\log p(y|x_S) - \lambda |S|].
\end{equation}
\begin{theorem}
For any $p(y|x_S)$, the $q(S|x,y)$ that maximizes \Cref{eqn:appendix_general_rationale} is a point-mass.   
\end{theorem}
\begin{proof}
Denote by $g(x,y, S) = \log p(y|x_S) -\lambda |S|$:
\begin{align*}
\max_q \ \mathbb E_{x,y \sim F} &\mathbb E_{S \sim q(S|x,y)}[g(x,y,S)] \\
&\leq \max_q \mathbb E_{x,y \sim F} \max_S [g(x,y,S)] \\
&=\mathbb E_{x,y \sim F} \max_S [g(x,y,S)].
\end{align*}
The inequality uses the fact that the expectation of a random variable is bounded by its maximum value. When $q(S|x,y)$ is a point-mass at $\argmax_{S}[g(x,y,S)]$, the inequality becomes tight.
\end{proof}

The fact that the optimal rationale is deterministic for each example justifies using combinatorial strategies such as our objective in \Cref{eqn:combinatorial_rationale}.

\section{Efficiency}
\label{sec:appendix_computational_comparisons}

In \Cref{tab:theoretical_complexity}, we provide a detailed version of our complexity analysis from \Cref{sec:greedy}: For transformers, greedy rationalization can be performed at no extra asymptotic complexity if the rationale length is $O(t^{1/3})$ for a sequence length $t$.

\begin{table}
  \small
  \centering
  \begin{tabular}{ l c c c }
   \toprule
   Step & Complexity & Evaluations & Total \\
    \midrule
  $1$ & $1^2$ & $t$ & $1^2t$ \\
  $2$ & $2^2$ & $t-1$ & $2^2(t-1)$\\
  $\vdots$ & $\vdots$ & $\vdots$ & $\vdots$ \\
  $O(t^{1/3})$ & $O(t^{2/3})$ & $O(t)$ & $O(t^{5/3})$ \\
  \midrule
  Total &  & & $O(t^2)$\\
    \bottomrule
 \end{tabular}
 \caption{For transformers, the asymptotic complexity of greedy rationalization matches the asymptotic complexity of forming a single prediction on the full sequence, as long as the rationale size is $O(t^{1/3})$ for a sequence of length $t$.}
 \label{tab:theoretical_complexity}
\end{table}

We evaluate the computational efficiency of greedy rationalization in \Cref{tab:runtime_comparisons}. We compare greedy rationalization to an exhaustive search, which enumerates all possible context context subsets from shortest to longest to optimize \Cref{eqn:combinatorial_rationale}. To show the efficiency of evaluating transformers on arbitrarily sized inputs, we also compare to a version of greedy rationalization that evaluates a transformer on the full input. To make predictions on sparse subsets, this approach masks tokens that aren't in a candidate rationale during each attention step. In contrast, the efficient version of greedy rationalization only takes as input the tokens in the candidate rationale, so there is no need for masking.

We perform these comparisons on the templated analogies dataset of \citet{mikolov2013representations}. We use GPT-2 Large as our sequence model \citep{radford2019language} and perform each method on a single GPU. We compare the two greedy rationalization approaches for all of the examples for which the full model predicts the templated output. Since exhaustive search is intractable, we cannot perform it on every example due to computational constraints. Thus, we only run exhaustive search on examples where the optimal rationale has 6 or less tokens. In reality, the average runtime for exhaustive search is larger than the listed one.

\begin{table}
  \small
  \renewcommand{\arraystretch}{1.0}  
  \centering
  \begin{tabular}{ l c}
   \toprule 
  \textbf{Method} & \textbf{Time (s)} \\
    \midrule
  Exhaustive search & >60 \\
  Greedy rationalization with full inputs & 1.22 \\
  Greedy rationalization with sparse inputs & 0.30 \\
    \bottomrule
 \end{tabular}
 \caption{Greedy rationalization is efficient, especially when evaluating transformers on sparse inputs. We report the average wall clock time in seconds for finding rationales on the templated analogies dataset of \citet{mikolov2013representations}. 
 We cannot complete exhaustive search for the longer examples, so in reality the average runtime is larger than the listed one.}
 \label{tab:runtime_comparisons}
\end{table}

\section{Training and Fine-Tuning Details}
\label{sec:appendix_finetuning}

Our experiments consist of three models and datasets:
a transformer decoder \citep{vaswani2017attention} trained on a majority-class language, GPT-2 \citep{radford2019language} fine-tuned on Open WebText \citep{gokaslan2019openweb}, and a transformer machine translation model trained and fine-tuned with word dropout on IWSLT14 De-En \citep{cettolo2014report}.

For the majority-class language, we generate the dataset as described in \Cref{sub:modularity}. We include 50,000 examples in the training set, 5,000 in the validation set, and 5,000 in the test set.

We use a 4-layer transformer decoder with 2 attention heads per layer. We use an embedding dimension of 64, and a hidden dimension of 256 for the feedforward layers. This corresponds to 200,000 parameters. We train with $0.1$ weight dropout, and optimize using Adam \citep{kingma2014adam} with a learning rate of $0.005$ and an inverse square root learning rate scheduler. We use a warmup period of 4000 steps and an initial warmup learning rate of $10^{-7}$. We include a maximum of 64,000 tokens in each batch. We implement this model in Fairseq \citep{ott2019fairseq}. 

To approximate the compatibility objective in \Cref{eqn:word_dropout_objective}, we train with varying amounts of word dropout. In practice, this amounts to masking out each token we drop out at each attention layer. We use two levels of word dropout in \Cref{fig:majority_class_calibration}; none (which corresponds to training with the standard maximum likelihood objective in \Cref{eqn:setup_empirical_risk}) and $0.5$. We train each model on a single GPU. Each model takes less than 20,000 steps to converge, less than 90 minutes. \Cref{tab:finetuned_ppl} verifies that fine-tuning with word dropout does not hurt the heldout perplexity.

To fine-tune GPT-2, we use the pretrained GPT-2 Large model available on Hugging Face \citep{wolf2019huggingface}. This model has 774M parameters. We don't change any of the model settings when we fine-tune. Sampling context subsets uniformly at random as stated in the objective in \Cref{eqn:word_dropout_objective} results in a distribution of subsets heavily skewed towards those containing half the words in the sequence. This is fine for the majority-class language, since each sequence contains less than 20 tokens and thus all possible context sizes will be seen during training. However, GPT-2's sequence length is 1,024. 99\% of the time, sampling from the objective as stated would result in contexts with size 464-560. Notably, the probability of a context with less than 10 tokens is less than $10^{-284}$.

We make two adjustments to make sure the model is trained on both small and large subsets. With probability $0.5$, we condition on the full context. With the remaining $0.5$, we first randomly sample context sizes uniformly at random from $1$ to the sequence length. We then sample a random context subset of this size. This guarantees that all possible sequence lengths will be seen during training.

Since the WebText dataset used to train GPT-2 is not publicly available, we use Open WebText \citep{gokaslan2019openweb}, an open source reproduction effort. The corpus is in English. Rather than using the entire dataset, we take ``Subset 9'' and use the first 163M words. Our validation set is also from this subset and contains 160,000 words. We use a test set of 300,000 words from a different subset.

We fine-tune GPT-2 Large using Adam. We use a constant learning rate of $0.0001$, using a single batch per training step. We stop training after 62,500 steps. This takes 15 hours on a single GPU. \Cref{tab:finetuned_ppl} shows that fine-tuning with word dropout actually improves the heldout perplexity, although we believe that the improvement is due to our test set bearing more resemblance to the fine-tuning set than to the pretraining set. 

We use a standard transformer encoder/decoder to train a machine translation model on IWSLT14 De-En \citep{cettolo2014report}. We follow the preprocessing and model architecture recommended by Fairseq.\footnote{\url{https://github.com/pytorch/fairseq/tree/master/examples/translation}} The training set has 160,239 translation pairs, the validation set has 7,283, and the test set has 6,750. 

As for the model, both the encoder and decoder are transformers with 6 layers, 4 attention heads per layer, 512 embedding dimensions, and 1024 feedforward dimensions. This corresponds to 40M parameters. We train with $0.3$ weight dropout and $0.1$ label smoothing, using 4,096 tokens for each train step. We train with Adam with a learning rate of $5\times 10^{-4}$ and use an inverse square root learning rate scheduler with 4,000 warmup steps. 

When we fine-tune for compatibility, we again condition on the full context with probability $0.5$. With the remaining probability, we drop out each source and target token independently at each attention head with probability $1-1/T$, where $T$ is the sequence length (so the dropout probability varies for the source and target sequence). Although we drop out different tokens at each attention head of a layer, we make sure that the same tokens are dropped out at each layer. Our word dropout procedure ensures that our objective will be trained on small contexts since rationales for machine translation are typically very sparse. We fine-tune using Adam with a constant learning rate of $10^{-5}$ for 410,000 steps. The heldout BLEU scores of both models are equal; see \Cref{tab:finetuned_ppl}.

\begin{table*}
  \small
  \renewcommand{\arraystretch}{1.0}  
  \centering
  \begin{tabular}{ l l l c c}
   \toprule 
  Model & Dataset & Evaluation Metric &  Standard Training & Compatible Training \\
    \midrule
  Transformer decoder & Majority-Class & Perplexity & 1.8 & 1.8 \\
  GPT-2 & Open WebText & Perplexity & 18.3 & 17.1 \\
  Transformer encoder/decoder & IWSLT14 EnDe & BLEU & 34.8 & 34.8 \\
    \bottomrule
 \end{tabular}
 \caption{Fine-tuning for compatibility does not hurt heldout performance. The first two rows are language models and the evaluation metric is heldout perplexity; the last row is machine translation, for which the evaluation metric is BLEU.}
 \label{tab:finetuned_ppl}
\end{table*}

\section{Experimental Details}
\label{sec:appendix_experimental_details}

\subsection{Long-Range Agreement}

\Cref{tab:analogies_template} contains the template we used for the set of experiments containing the analogies from \citet{mikolov2013representations}. To avoid rationales containing partial antecedents, we only include examples where both words in the analogy correspond to single word-pieces using GPT-2's tokenizer. Since it only makes sense to rationalize correct predictions, we also only include the examples where GPT-2 correctly completes the analogy. In total, this results in 175 examples. Of these, we randomly sample 50 to perform exhaustive search, which we use to compute the approximation ratio of each method. Since we cannot run exhaustive search when the minimal sufficient rationale is too large, we use the 40 that converge with rationales of length 6 or less. We use 100 steps to approximate the path integral for the integrated gradients baseline \citep{sundararajan2017axiomatic}.

To confirm that the baseline performances are not being hindered by fine-tuning for compatibility, we re-run the experiment for each rationalization method on the pretrained GPT-2 Large, without any fine-tuning. The results are depicted in \Cref{tab:agreement_comparisons_no_finetuning}. As expected, the baselines perform even worse when GPT-2 is not fine-tuned to form compatible distributions. We do not include comparisons to exhaustive rationales because it is computationally infeasible to run exhaustive search on incompatible models, since optimization takes much longer to converge.

\begin{table*}
  \normalsize
  \setlength{\tabcolsep}{10pt}  
  \renewcommand{\arraystretch}{1.0}  
  \centering
  \makebox[\textwidth][c]{
  \begin{tabular}{ l p{120mm} }
  Relationship & Example \\ 
    \toprule
  Capital countries & When my flight landed in \textbf{Greece}, I converted my currency and slowly fell asleep. (I had a terrifying dream about my grandmother, but that's a story for another time). I was staying in the capital, \textbf{Athens} \\
  Currency & As soon as I arrived in \textbf{Japan}, I checked into my hotel and took a long nap. (I had finally finished the book I was reading and it was amazing). I had to figure out the exchange rate to the local currency, which is apparently called the \textbf{yen}\\
  City in state & As soon as I arrived in \textbf{Florida}, I checked into my hotel and watched a movie before falling asleep. (I had a great call with my husband, although I wish it were longer). I was staying in my favorite city, \textbf{Miami} \\
  Family & I initially invited my \textbf{uncles}, who gladly accepted my invitation. (My favorite song just came on, so I was able to relax). When I learned that women were allowed, I went ahead and also invited my \textbf{aunts} \\
  Opposite & I thought it was \textbf{pleasant}. (Just then an ad came on the TV, but that's irrelevant). It was the opposite of that: it was \textbf{unpleasant} \\
  Comparative & I knew it was \textbf{tall}, but that's before I saw it in person. (Just then I thought about my ex-wife, but I had to stop thinking about her). When I did end up seeing it in person, it was even \textbf{taller} \\
  Superlative & I thought it would be the \textbf{smallest} thing I'd ever encounter. (I tried to ignore my phone vibrating in my pocket). But when I did end up encountering it, it turned out it wasn't so \textbf{small}\\
  Present participle & Every other day, it started \textbf{working} in the morning. (I tried to remember the name of the woman at the bar). But today, it did not \textbf{work} \\
  Nationality adjective & I had never been friends with any \textbf{French} people before. (The funniest thing happened to me the other day, but that's a story for another time). In fact, I had never even been to \textbf{France}\\
  Past tense & Although I \textbf{listened} yesterday, I had a million things to do today. (I suddenly felt a pinched nerve, so I made a mental note to get that checked out). So today I wouldn't have time to do any more \textbf{listen} \\
  Plural & I really wanted to buy the \textbf{computer}, more than I ever wanted to buy anything before. (I was also behind on my homework, but that's another story). So I went to the store and asked if they had any \textbf{computers} \\
  Plural verbs & I can usually \textbf{sing} by myself. (I was so behind on work but I tried to distract myself). Although it's so much better when someone else also \textbf{sings} \\
  \bottomrule
\end{tabular}}
 \caption{Template using analogies from \citet{mikolov2013representations}.}
 \label{tab:analogies_template}
\end{table*}

\begin{table}
  \small
  \centering
  \begin{tabular}{ l  c c c}
  \toprule 
  & Length & Ante & No D  \\
    \midrule
  Gradient norms & 24.8 & 1.0 & 0.08 \\
  Gradient x embedding & 41.1 & 0.99 & 0.00  \\
  Integrated gradients & 34.9 & 1.0 & 0.00  \\
  Attention rollout & 38.4 & 1.0 & 0.05 \\
  Last attention layer & 20.1 & 0.99 & 0.03 \\
  All attention layers & 19.5 & 1.0 & 0.02 \\
  Greedy & 13.1 & 1.0 & 0.30 \\
    \bottomrule
 \end{tabular}
 \caption{The performance of each rationalization method on the templated version of the analogies dataset \citep{mikolov2013representations} when we don't fine-tune for compatibility. As expected, fine-tuning for compatibility (\Cref{tab:agreement_comparisons}) improves performance across the board.}
 \label{tab:agreement_comparisons_no_finetuning}
\end{table}

\subsection{Machine Translation}

For the distractor experiment, we randomly concatenate 500 pairs of source and target sequences generated by our fine-tuned model on the test set. We evaluate rationales by counting how many times they ``cross over'' and contain words from the distractor sequence. We do not penalize rationales that include special tokens like the beginning of sentence or end of sentence tokens.

For the alignment experiment, we use a public corpus of annotated rationales.\footnote{\url{https://www-i6.informatik.rwth-aachen.de/goldAlignment/}} Not every word in the dataset has an alignment, and some words have multiple alignments. Although the human annotations are on word-level alignments, our machine translation models are trained on subwords, so the rationales contain subwords in addition to full words. To make these comparable to the human annotations, we define the rationale of a full target word to contain the union of the subword rationales. Since each source word may also be a subword, we also take the union of source words in a rationale. To calculate top-1 accuracy, we define the rationale for a full word to be accurate if the rationales for any of the subwords in the rationale contain any source subwords that are in the annotated alignment. 

The alignment dataset contains both "sure" and "possible" alignments. These are used to differentiate between different errors when calculating the alignment error rate \citep{och2000improved}. For the other metrics, we include both kinds of alignments as part of the annotated alignments.

For both machine translation experiments, we use 50 steps to approximate the path integral for the integrated gradients baseline \citep{sundararajan2017axiomatic}.

\subsection{Annotated Lambada}

We work with volunteers to annotate Lambada \citep{paperno2016lambada}. Each example requires two annotators: a selector, and a predictor. A selector's goal is to choose the most important words for predicting the final word of a passage, known as the target word. Predictors will only be seeing the words chosen by a selector, and their goal is to predict the final word of the passage.

The selector first takes a passage and ranks 15 words. The top-ranked word always needs to be the word before the final word of the passage. They cannot select the final word of the passage. Each of their selections needs to be a complete word. They cannot select the same word twice, and they need to use all 15 spots. They know that a predictor will be predicting words, one-at-a-time, using the order they create.

When a selector is finished ranking the top 15 words, a predictor begins by seeing the top ranked word. They use this to predict the last word. Words are revealed one-at-a-time in the order chosen by the selector. Predictors can see how much space is between the words that have been revealed. Predictors are not told if they predicted a word correctly; the goal of the exercise is to capture the predictor's true predictions, so if they knew that previous guesses were incorrect, they may use this information to guess a new word at every step. If a predictor is not able to guess the target word at the end of the exercise, we re-assign the example to another predictor.

\Cref{fig:appendix_lambada_selector_instructions} contains an example given to selectors. \Cref{fig:appendix_lambada_predictor_instructions} contains an example given to predictors.

\begin{figure*}
  \includegraphics[width=1.0\textwidth]{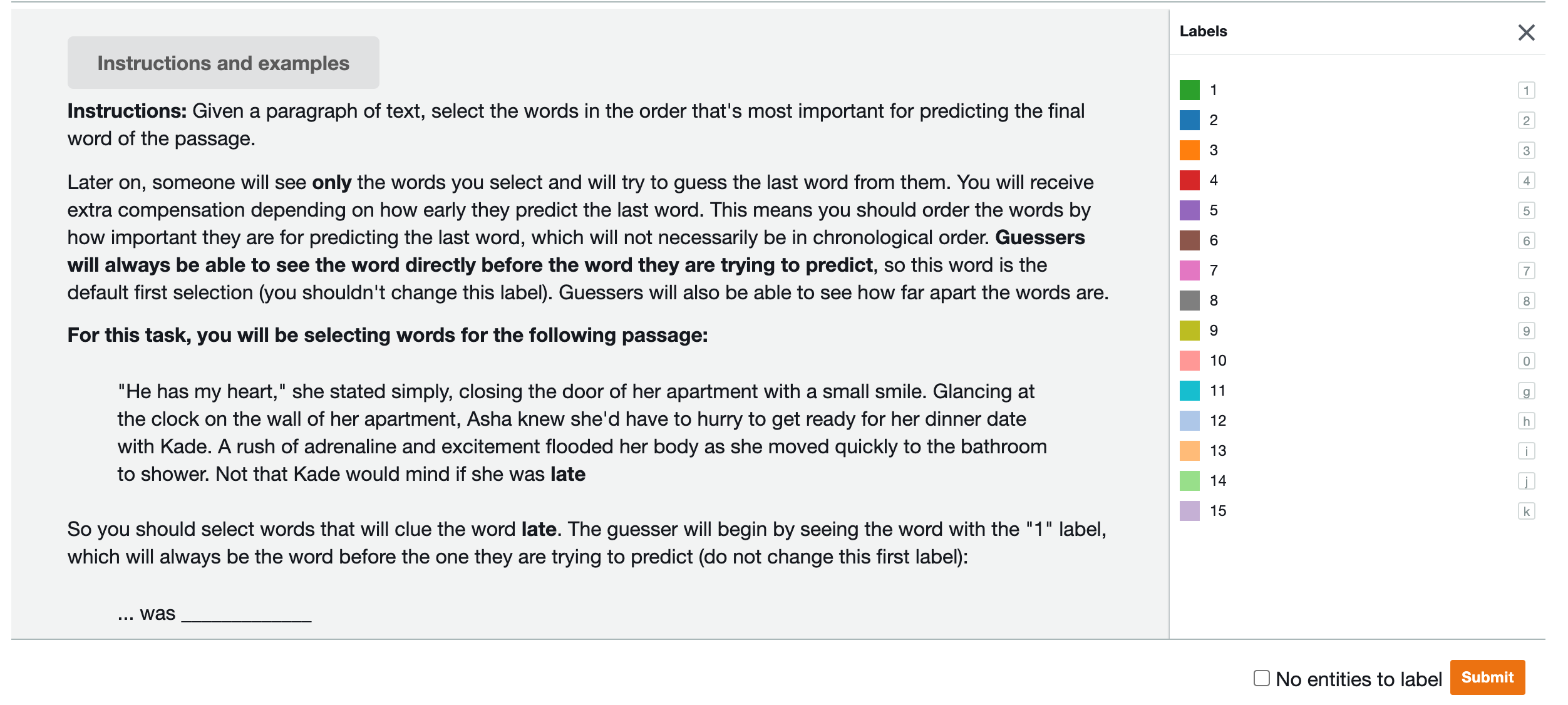}
  \includegraphics[width=1.0\textwidth]{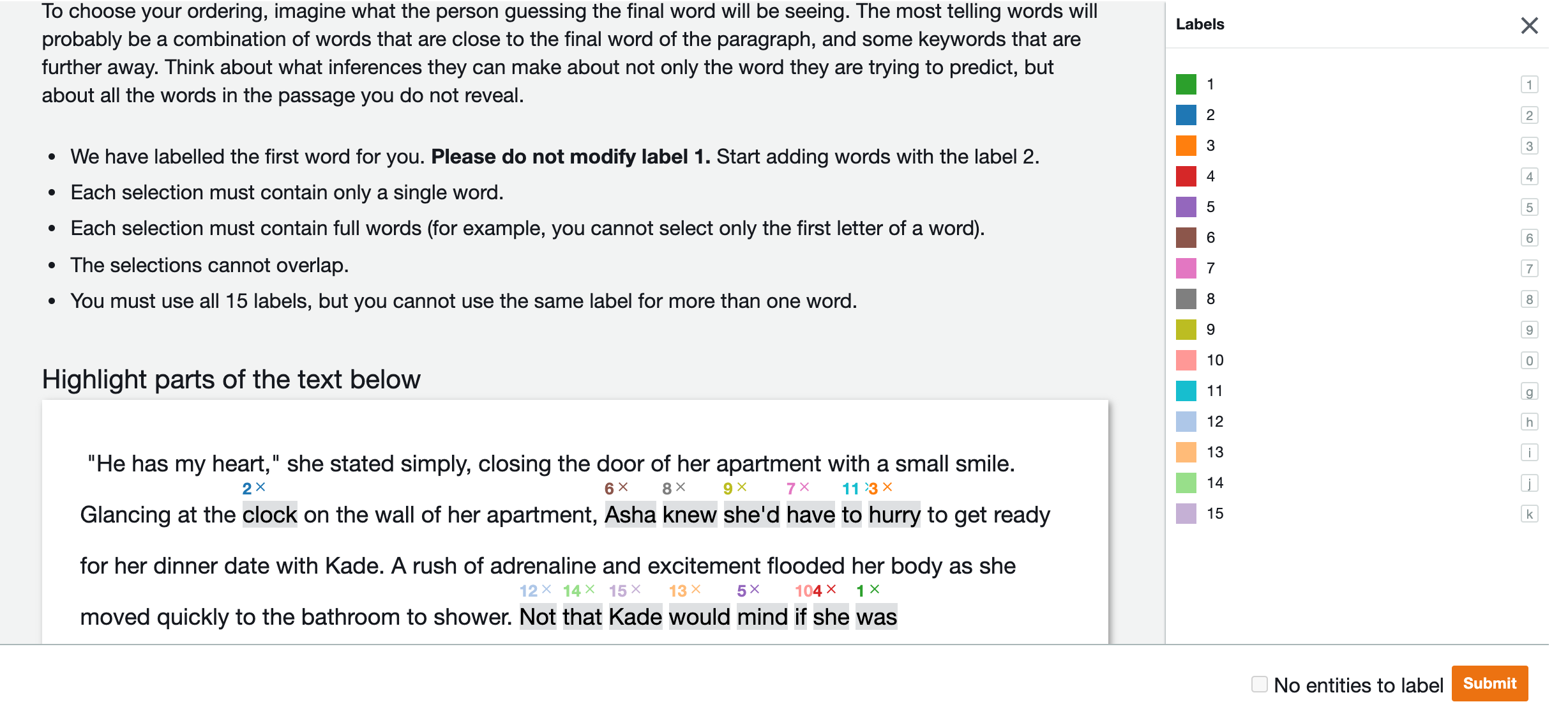}
  \caption{Sample instructions given to selectors to annotate Lambada.}
  \label{fig:appendix_lambada_selector_instructions}
\end{figure*}

\begin{figure*}
  \includegraphics[width=1.0\textwidth]{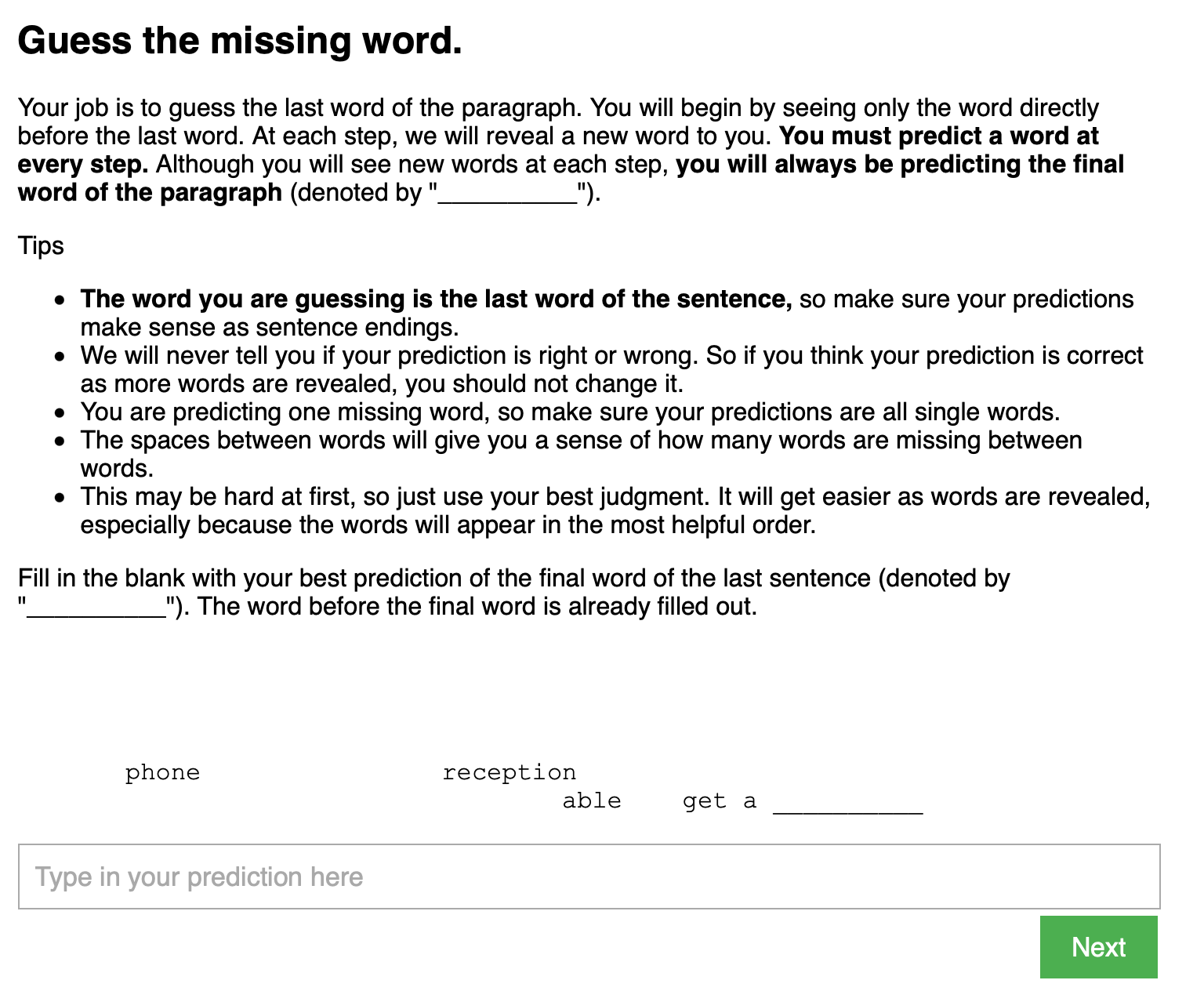}
  \caption{Sample instructions given to predictors to annotate Lambada.}
  \label{fig:appendix_lambada_predictor_instructions}
\end{figure*}

In total, we annotate 107 examples, and use all of them for the rationalization experiment. For each example, we define a human's rationale to be all the words that were revealed by the selector before the predictor first predicted the true target word or a synonym of it. The average rationale length is $6.0$.

To compare human rationales to those found by various methods, we first tokenize the text with GPT-2's tokenizer, and convert an annotated rationale to its set of corresponding subwords. Each method's rationale is also a set of subwords. We use set-comparison metrics like intersection over union (IOU) and F1 to compare the similarity of rationales. We use 100 steps to approximate the path integral for the integrated gradients baseline \citep{sundararajan2017axiomatic}.

\section{Qualitative Examples}
\label{sec:appendix_qualitative_examples}
\Cref{fig:appendix_lambada_examples} contains examples of rationales on our annotated Lambada dataset.

\begin{figure}
  \includegraphics[width=0.48\textwidth]{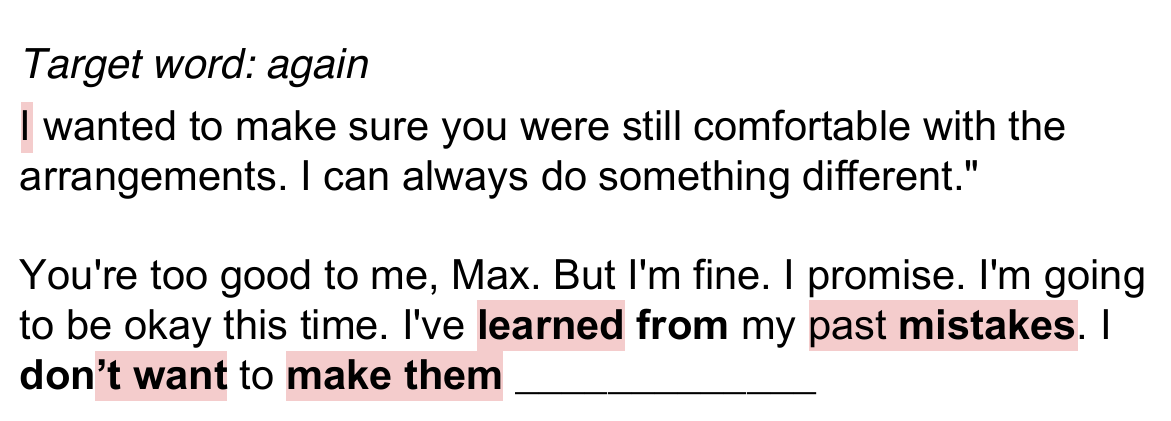}
  \includegraphics[width=0.48\textwidth]{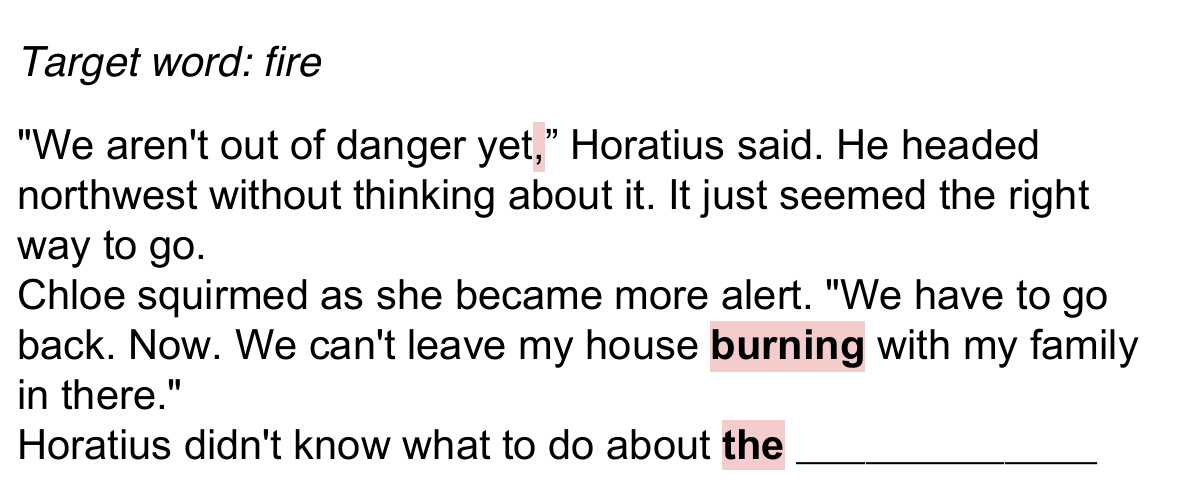}
  \includegraphics[width=0.48\textwidth]{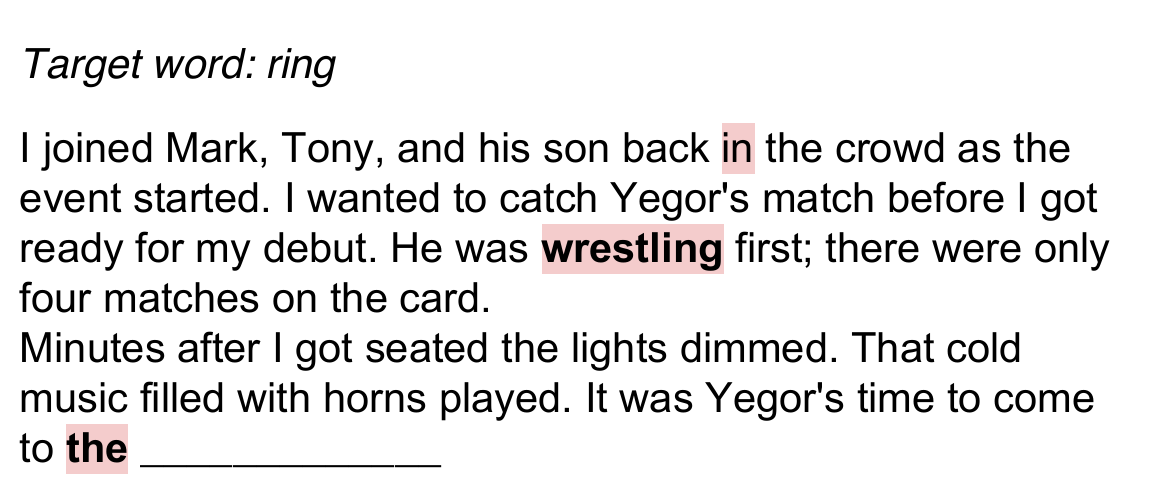}
  \includegraphics[width=0.48\textwidth]{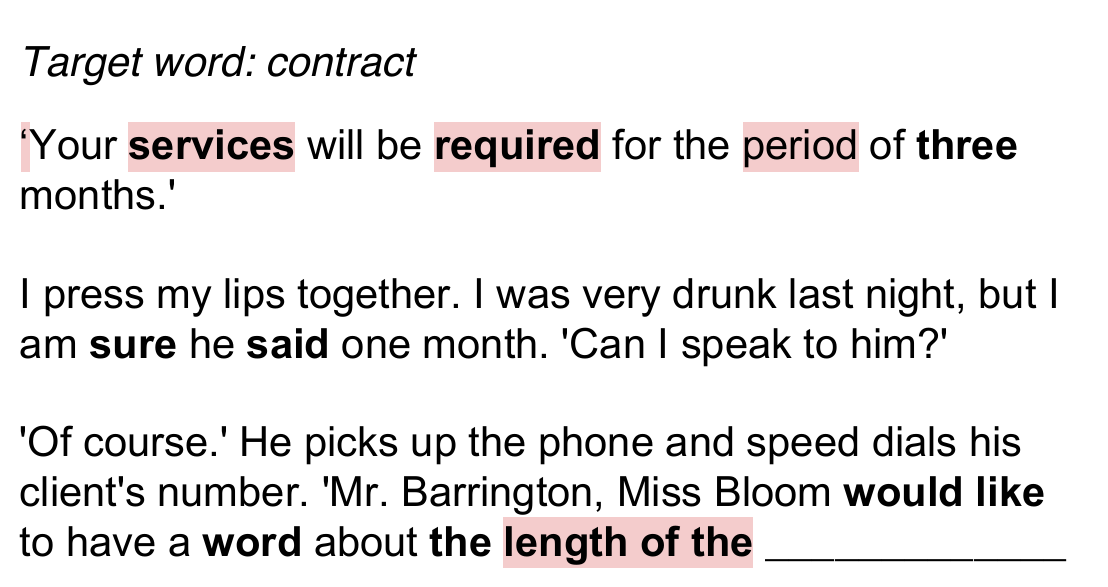}
  \caption{Sample rationales from our annotated Lambada dataset. \hl{Highlighted text} corresponds to greedy rationales, and \textbf{bolded text} corresponds to human annotated rationales.}
  \label{fig:appendix_lambada_examples}
\end{figure}

\end{document}